\documentclass[conference]{IEEEtran}
\IEEEoverridecommandlockouts
\usepackage{amsmath,amssymb,amsfonts}

\usepackage{csquotes}
\usepackage{enumitem,kantlipsum}
\usepackage{soul}

\usepackage{hyperref}
\usepackage{cleveref}
\usepackage[normalem]{ulem}
\usepackage{amsthm}
\usepackage{thmtools}
\usepackage{thm-restate}
\usepackage[english]{babel}
\newtheorem{proposition}{Proposition}
\newtheorem{theorem}{Theorem}
\newtheorem{assumption}{Assumption}
\newcounter{tempcounter}
\newtheorem{lemma}[theorem]{Lemma}
\newtheorem{definition}{Definition}
\newtheorem{corollary}{Corollary}
\newtheorem*{remark}{Remark}

\allowdisplaybreaks
\usepackage{amsfonts}
\usepackage{amssymb}
\usepackage{xcolor}         
\usepackage{amsmath}

\usepackage{enumitem}

\usepackage[normalem]{ulem}
\usepackage[ruled, vlined, linesnumbered]{algorithm2e}
\usepackage{enumitem,kantlipsum}
\usepackage{graphicx}

\usepackage[style=ieee, backend=biber]{biblatex}

\usepackage{algorithmic}
\usepackage{graphicx}
\usepackage{textcomp}
\usepackage{xcolor}
\def\BibTeX{{\rm B\kern-.05em{\sc i\kern-.025em b}\kern-.08em
    T\kern-.1667em\lower.7ex\hbox{E}\kern-.125emX}}
\begin{document}

\title{Bounded (O(1)) Regret Recommendation Learning via Synthetic Controls Oracle
\thanks{Identify applicable funding agency here. If none, delete this.}
}

\author{\IEEEauthorblockN{Enoch Hyunwook Kang}
\IEEEauthorblockA{
\textit{University of Washington}\\
}
\and
\IEEEauthorblockN{
P. R. Kumar}
\IEEEauthorblockA{
\textit{Texas A\&M University}\\
}
}

\maketitle

\begin{abstract}

In online exploration systems where users with fixed preferences repeatedly arrive, it has recently been shown that $O(1)$, i.e., bounded regret, can be achieved when the system is modeled as a linear contextual bandit. This result may be of interest for recommender systems, where the popularity of their items is often short-lived, as the exploration itself may be completed quickly before potential long-run non-stationarities come into play. However, in practice, exact knowledge of the linear model is difficult to justify. Furthermore, potential existence of unobservable covariates, uneven user arrival rates, interpretation of the necessary rank condition, and users opting out of private data tracking all need to be addressed for practical recommender system applications. In this work, we conduct a theoretical study to address all these issues while still achieving bounded regret. Aside from proof techniques, the key differentiating assumption we make here is the presence of effective Synthetic Control Methods (SCM), which are shown to be a practical relaxation of the exact linear model knowledge assumption. We verify our theoretical bounded regret result using a minimal simulation experiment.

\end{abstract}


\section{Introduction}\label{sec:intro}
In many modern personalization systems (e.g., video/music recommendation systems), users with highly heterogeneous preferences arrive sequentially and repeatedly to choose an item. In this context, online learning models have been increasingly used for their ability to address potential presentation bias \cite{steck2021deep} which may occur through repeated data collection - learning feedback loop when using typical matrix methods \cite{fleder2009blockbuster, lee2019recommender}.

Among online learning models, contextual bandit models \cite{bendada2020carousel, tang2014ensemble, tang2015personalized, zhou2016latent, kong2020sublinear, bietti2021contextual, foster2020beyond} are often used when it is reasonable to make two assumptions on the reward model. The first assumption is that a user's observable covariates (e.g., demographics, age, sex, etc.), user's arm (=item) choice, and unknown model parameters jointly determine the stochastic reward model. (When these unknown parameters are not shared among arms, we call this model a ``disjoint" contextual bandit model.) The vector of observed covariates is called the context. The second assumption is that the distribution of the stochastic reward from an arm with a given context does not change over time. (Each user's context itself may change over time in the long run.) The objective of the contextual bandit model is often to minimize the order of cumulative regret (=total loss of welfare) as a function of time $T$ \cite{li2010contextual, dimakopoulou2017estimation, chu2011contextual, foster2020instance}, as many recommender systems prioritize user welfare.

In the special case when users arrive repeatedly and each user's context can be considered fixed over the short run (e.g., a music listener's genre preference does not change after a few songs), the best known regret bound is constant, i.e., O(1) \cite{hao2020adaptive}, for the class of models called linear contextual bandit models \cite{li2010contextual, chu2011contextual, dimakopoulou2017estimation}. In typical linear contextual bandit model settings, it is assumed that there is a known function $\phi_m$ that maps a context into a vector representation that is linearly related to the mean of the reward from arm $m$. It has been shown in \cite{hao2020adaptive, papini2021leveraging} that when user arrivals are modeled as random sampling from a fixed distribution, boundedness of regret can be attained if and only if the representation vectors jointly satisfy a full rank condition (details in Section \ref{sec:haopapinireview}).

The boundedness of regret for linear contextual bandit models is an attractive result for many recommender system applications. Popularities of video topics on video platforms are only ephemeral; topics may lose their timeliness before they get old. However, this violates the second assumption of contextual bandits. Addressing such potential non-stationarities, either by explicitly modeling non-stationarities or by addressing worst-case scenarios, has been an active area of study in the bandit literature \cite{besbes2014stochastic, russac2019weighted, russac2020algorithms, zhao2020simple, luo2018efficient, chen2019new,
jia2023smooth, qin2022adaptivity, fiez2022adaptive}. The bounded regret approach attempts to address non-stationarities by completing exploration quickly enough before long-run non-stationarities kick in.

Despite the attractiveness of the $O(1)$ regret result in contextual bandits, it is not clear how the conditions described above can be justified for practical recommender systems. Specifically, there are five issues that need to be addressed:
\begin{enumerate}[wide, labelindent=0pt]
\item [\textit{Issue (1):}] Due to the potential existence of unobservable covariates \cite{rosenbaum1983assessing},  known context information (along with arm choice) may not fully determine the stochastic reward. This violates the contextual bandit model assumption.
\item [\textit{Issue (2):}] The existence of a linear representation, a well-justified assumption \cite{papini2021leveraging}, does not justify the assumption that the linear representation function is exactly known. This violates the linear contextual bandit model assumption.
\item [\textit{Issue (3)}:] User arrivals may be far from i.i.d. sampling; they may even be of different orders, such as $\ln T$ and $T$.
\item [\textit{Issue (4)}:] \textit{Disjoint} linear contextual bandit models (where unknown parameters are not common to arms) are widely used \cite{li2010contextual, dimakopoulou2017estimation} for recommender systems. However, the condition required for the disjoint case to achieve bounded regret \cite{wu2020stochastic, papini2021leveraging} is not an easily operationalizable condition (details in Section \ref{sec:haopapinireview})
\item [\textit{Issue (5):}] A user's context and the rewards she observes may remain private information if she opts out of tracking.
\end{enumerate}

In this paper, we focus on the theoretical study of addressing these issues either by relaxing or by justifying the assumptions to show that we can consider using bounded O(1) regret methods for recommender systems applications.

\begin{enumerate}[wide, labelindent=0pt]
\item [-] For issues (1) and (2), we assume the existence of Synthetic Control method (SCM) \cite{abadie2010synthetic, abadie2011bias, doudchenko2016balancing, amjad2018robust, abadie2021using, ben2021augmented, ferman2021properties} we can use, which is ``arguably the most important innovation in the policy evaluation literature in the last 15 years'' \cite{athey2017state}. We observe that what is achieved in Synthetic Control methods is exactly equivalent to a relaxation of the assumption that the linear representation function is precisely known\footnote{Most Synthetic Control methods address linear factor model settings, which are non-stationary generalizations of disjoint linear contextual bandit model settings \cite{dimakopoulou2017estimation}. While our setting considers sequential arrivals, the stationarity of the setting allows the application of synthetic control methods.}; this resolves issue (2). As Synthetic Control methods can address unobserved covariates in the long run \cite{abadie2010synthetic}, issue (1) is also resolved.
\item [-] For issue (3), our condition only requires some of the users to have similar order of arrival rates.
\item [-] For issue (4), we provide an operationalizable condition that requires the user set size to be larger than $|M|\ln |M|$ (this value may be larger under non-uniform preferences among users over arms).
\item [-] For issue (5), we show users' strong incentive to opt in and comply with the recommendations.
\end{enumerate}
The rest of the paper is organized as follows. 
We provide the relevant background on bounded regret results and SCM in Section \ref{sec:preliminaries}. Then we present our main model, the main algorithm we call Counterfactual-UCB (CFUCB), and its bounded regret analysis in Sections \ref{sec:mainmodel}, \ref{sec:CFUCB}, and \ref{sec:AnalysisCFUCB}. Finally, we further validate the proposed theory via a minimalistic simulation experiment in Section \ref{sec:simulation}. After reaching the conclusion, Section \ref{sec:related} briefly discusses related previous works.



\section{Preliminaries}\label{sec:preliminaries}

\subsection{Bounded regret results for (disjoint) linear contextual bandit models}\label{sec:haopapinireview}
In this section, we review the problem settings and conditions for which bounded regret can be attained for non-disjoint \cite{hao2020adaptive} and disjoint \cite{wu2020stochastic} linear contextual bandit models. In Section \ref{sec:mainmodel}, we discuss how they can be relaxed or justified for practical recommender systems. 

Let $A$ be the set of users and $M$ the set of arms. Each user 
$j \in A$ is associated with a context vector $\mathbf{x}_j \in \mathbb{R}^k$, where $k$ is the number of observable covariates for each user. 

\noindent {\textbf{Rewards model}}  \space  In the setting of \cite{hao2020adaptive}, every time user $j$ arrives and pulls arm $m$ the user receives a reward $\phi_m(\mathbf{x}_j)'\beta+\epsilon$, where  $\{\phi_m: \mathbb{R}^k \mapsto \mathbb{R}^d\}_{m\in M}$ are linear representation functions that are assumed to be \textit{precisely known}, $\beta$ is a common parameter vector of dimension $d$ that is shared across the arms, and $\epsilon$ is a i.i.d. zero-mean random noise that follows a sub-Gaussian distribution with variance proxy $\sigma^2$. In the setting of disjoint linear contextual bandits \cite{wu2020stochastic}, the reward equation is $\phi_m(\mathbf{x}_j)'\beta_m+\epsilon$ where $\beta_m$ is an
arm-specific parameter vector for arm $m$. 

\noindent {\textbf{User arrivals}}  \space In the settings of \cite{hao2020adaptive, wu2020stochastic},  user arrivals are modeled as a result of repeated i.i.d. random sampling according to a fixed distribution over $A$. Note that this user arrival model is equivalent to independent repeated user arrivals with exponential inter-arrival times \cite{grimmett2020probability}.

\noindent {\textbf{Condition for bounded regret}}  \space It is shown in  \cite{hao2020adaptive, papini2021leveraging}  that bounded regret can be achieved in this setting if and only if $\left\{\phi_{m_{j*}}(x_{j}) \mid j \in A \right\}$ spans $\mathbb{R}^d$, where $m_{j*} $ is the optimal arm for user $j$, i.e., $m_{j*}=\operatorname{argmax}_{m \in M}\phi_m(\mathbf{x}_j)'\theta$. In the disjoint case, bounded regret can be achieved if and only if $\left\{\phi_{m_{j*}}(\mathbf{x}_{j}) \mid j \in A_m \right\}$ spans $\mathbb{R}^d$ for each $m \in M$, where $A_m$ is the set of users whose optimal arm is $m$, i.e., $A_m=\{j\in A: m_{j*} = m\}$ \cite{wu2020stochastic, papini2021leveraging}. Since $k$ generic randomly generated vectors in $\mathbb{R}^d$ with $d \geq k$ are almost surely linearly independent, this condition can simply be rewritten as $|A_m|\ge d$ for $m\in M$. 

\subsection{Synthetic Control Methods (SCM)}\label{sec:SCMreview}
Synthetic Control Methods (SCM) have been one of the most actively studied areas of econometrics \cite{athey2017state, abadie2021using}. They can be described as an observational method of finding a linear combination to synthetically construct a user $j\in A$ from other users in $E\subset A\setminus \{j\}$ using their contexts and previous data.  While the coefficients of the linear combination were  constrained to be non-negative and sum to one in the vanilla SCM \cite{abadie2010synthetic, abadie2011bias}, recent advances in SCM effectively relax these constraints \cite{doudchenko2016balancing, amjad2018robust, abadie2021using, ben2021augmented, ferman2021properties}. Throughout, we will consider this more relaxed version of SCMs. 
\begin{definition}[Synthetic Control Method (SCM)]\label{def:SCM} Suppose that we are given context vectors $\{\mathbf{x}_i\}_{i\in A}$ and previous reward histories $\{\mathbf{h}_i\}_{i\in A}$. Denote the one and only arm as arm $1$. A Synthetic Control Method (SCM) is a method that, for given large enough $E\subseteq A$ and user $j\notin E$, takes $\{\mathbf{x}_i\}_{i\in E\cup\{j\}}$ and $\{\mathbf{h}_i\}_{i\in E\cup\{j\}}$ as inputs and outputs $\{a_{ji}\}_{i\in E}$ that satisfies $\mu_{j1}=\sum_{i\in E} a_{ji} \mu_{i1}$, where $\mu_{j1}$ describes user $j$'s mean reward from arm $1$.
\end{definition}

\begin{lemma}[Abadie et al., 2010 \cite{abadie2010synthetic}]\label{lemma:SCMkeyresult} Given long enough $\{\mathbf{h}_i\}_{i\in E\cup\{j\}}$, SCM can infer the linear combination coefficients $\{a_{ji}\}_{i\in E}$ described in Definition \ref{def:SCM} as precise as we want, even in the presence of the unobservable covariates\footnote{SCM typically considers the reward model $Y_{j}(k)
=\phi(\mathbf{x}_j)'\beta(t)+\psi(\mathbf{y}_j)'\lambda(t)+\epsilon(k)=\mu_{j}+\epsilon(k)$, which is called the linear factor model. Abadie et al. 2010 \cite{abadie2010synthetic} shows Lemma \ref{lemma:SCMkeyresult} for the linear factor model.}. 
\end{lemma}

\section{The main model}\label{sec:mainmodel}  In this section, we introduce the main model considered in this paper.
We denote the set of users by $A$ and the set of arms by $M$. We further denote by $A_+$ the subset of users who opt in for the revelation of their private data.
That is, the recommender knows $\mathbf{x}_j \in \mathbb{R}^k$ of user $j \in A_+$ and the rewards user $j\in A_+$ receives.

\subsection{Rewards model and objective}\label{sec:ourmodel}

Let $Y_{j,m}(k)$ denote the reward obtained from user $j$'s $k$th pull of arm $m$. We first consider the reward model $Y_{j,m}(k)
=\phi_m(\mathbf{x}_j)'\beta_m+\psi_m(\mathbf{y}_j)'\lambda_m+\epsilon(k)=\mu_{j,m}+\epsilon(k)$, the multi-arm extension of the static version of the model$^2$ usually considered by SCM methods. 

For each user $j\in A$, define $m_j^* \in M$ as an optimal arm that satisfies $\mu_{j, m^*_j}\geq \mu_{j,m} \; \; \forall m \in M$,
and $\Delta_{j,m}:=\mu_{j, m^*_j}-\mu_{j,m}$ as the instantaneous pseudo-regret of using arm $m$. Denote the arm pulled by user $j$ at its $n$-th arrival by $m_j(n)$. Let $N_j(t)$ be the random variable indicating the total number of arrivals of
user $j$ until time $t$. Then the finite time pseudo-regret of user $j$ until time $T$ is $Regret_{j}(T):=\sum_{n=1}^{N_j(T)} \Delta_{j, m_j(n)}=\sum_{n=1}^{N_j(T)}(\mu_{j, m^*_{j}}-\mu_{j, m_j(n)})$, which we will simply abbreviate as ``regret." The system's total regret is $\sum_{j\in A} Regret_{j}(T)$.

Let $\mathbf{h}_{j,m}$ denote the previous history of rewards for user $j\in A_+$ from the arm $m\in M$. Definition \ref{def:SCM} and Lemma \ref{lemma:SCMkeyresult} in Section \ref{sec:SCMreview} allows us to make the following assumption: 

\begin{assumption}[Synthetic Control Oracle (SCO)]\label{assumption: SCO} Fix an arm $m\in M$. For any $E\subseteq A_+$ and $j\in E^c\cap A_+$ that satisfies $rank(\{\phi_m(\mathbf{x}_i)\})_{i\in E}\geq dim(\phi_m(\mathbf{x}_j))$, there is a Synthetic Control Oracle (SCO) that takes $(\{\mathbf{x_i}\}_{i\in E\cup\{j\}},\{\mathbf{h}_{i,m}\}_{i\in E\cup\{j\}})$ as its input and outputs $\{a_{ji}^{(m)}\}_{i\in E}$ that satisfies $\mu_{jm}=\sum_{i\in E} a_{ji}^{(m)} \mu_{im}$, regardless of the unobservable covariates $\{\mathbf{y}_i\}_{i\in E\cup\{j\}}$.
\end{assumption}


The following lemma shows that the SCO assumption is a relaxation of the requirement for knowledge of a precise linear model knowledge in typical linear contextual bandit models:

\begin{lemma}\label{lemma:SCOisrelaxation}  Assumption \ref{assumption: SCO} is a relaxation of the assumptions that (i) there are no unobserved covariates, and (ii) that the linear model is known to the recommender.
\end{lemma}
\begin{proof}[Proof of Lemma \ref{lemma:SCOisrelaxation}] From Lemma \ref{lemma:SCMkeyresult}, it is immediate that Assumption \ref{assumption: SCO} is a relaxation of the $(i)$ part assumption that there are no unobserved covariates. For the $(ii)$ part, for any $E\subseteq A$ with  $rank(\{\phi_m(\mathbf{x}_j)\})_{j\in E}\geq dim(\phi_m(\mathbf{x}_j))$ and $j\notin E$, we can find $\{a_{ji}^{(m)}\}_{i\in E}$ such that $\phi_m(\mathbf{x}_{j}) =\sum_{i\in E} a_{ji}^{(m)} \phi_m(\mathbf{x}_{i})$. This implies that $\mu_{j,m} = \phi_m(\mathbf{x}_{j})'\beta_m =\sum_{i\in E} a_{ji}^{(m)} \phi_m(\mathbf{x}_{i})'\beta_m=\sum_{i\in E} a_{ji}^{(m)}\mu_{i,m}$.
\end{proof}

In many practical recommender systems, it is not necessary to consider a separate user context representation function $\phi_m$ for each arm $m\in M$. For example, in movie recommendation problem, user context representation may represent a user's affinity for different genres, how much the user values plot complexity, and the user's preference for movies with a specific mood (e.g., light-hearted, serious, or thought-provoking), all of which are characteristics not specific to a particular movie. This leads us to consider a simplified model  $Y_{i,m}(k)
=\phi(\mathbf{x}_j)'\beta_m+\psi(\mathbf{y}_j)'\lambda_m+\epsilon(k)=\mu_{j,m}+\epsilon(k)$. Below is the resulting simplified form of Assumption \ref{assumption: SCO}.

\setcounter{tempcounter}{\value{assumption}}

\setcounter{assumption}{\numexpr\value{tempcounter}-1\relax}

\begin{assumption}[Synthetic Control Oracle (SCO) in\textit{ simplified }reward model]
For any $E\subseteq A_+$ and $j\in E^c\cap A_+$ that satisfies $rank(\{\phi(\mathbf{x}_i)\})_{i\in E}\geq dim(\phi(\mathbf{x}_j))$, there is a Synthetic Control Oracle (SCO) that takes $(\{\mathbf{x_i}\}_{i\in E\cup\{j\}},\{\mathbf{h}_{i,m}\}_{i\in E\cup\{j\}})$ as its input and outputs $\{a_{ji}\}_{i\in E}$ that satisfies $\mu_{jm}=\sum_{i\in E} a_{ji} \mu_{im}$ for any $m$, regardless of the unobservable covariates $\{\mathbf{y}_i\}_{i\in E\cup\{j\}}$.

\end{assumption}

The simplified SCO is useful in practice as it can generalize experiences among arms.
In Spotify, for example, there are more than 60,000 songs (=arms) newly registered each day \cite{Ingham2021}; given this simplified disjoint version Assumption \ref{assumption: SCO}, the SCO in Spotify's case can take previous user experiences from existing songs as its input and output the linear combination coefficients that can be used for the future exploration of newly registered songs.

\subsection{User arrivals }\label{sec:arrivalModel} 
We generalize the arrival model in Section \ref{sec:haopapinireview} to allow for users with arrival rates of different orders.
Let $S_j(n)$ be the random variable indicating user $j$'s $n$th arrival time, and $F^{(n)}_j(t):=P(\{S_j(n)\le t\})$. For $i, j\in A$, define $q_{ij,m}(x):=-\frac{B}{A}  \mathcal{W}_{-1}\left( -\frac{A}{B} (\frac{x}{d})^{-\frac{C}{B}} \right)$, where $\mathcal{W}_{-1}$ is the lower branch Lambert W-function \cite{corless1996lambert}, where $A=1$, $B=\sum_{n\neq m}\frac{16 }{{\Delta_{i,n}}  ^2}  $ and $C=\frac{16 c^2_{m,t}}{{\Delta_{j,m}}  ^2}$. (Remark: $q_{ij,m}(x)$ increases faster than $\ln x$ but slower than $x$ - see Appendix \ref{sec:qij}).

\begin{assumption}\label{ass: Arrivalcondition} Consider users $E_{j,m} := \{i \in A\setminus \{j\}: \lim\sup_{n \to \infty} \frac{\int_0^{+\infty}P(N_i(t)< q_{ij,m}(N_j(t)) dF_j^{(n)}(t)}{\frac{1}{n^2}} < + \infty \}.$
We assume that $|E_{j,m}\cap A_+ \cap A_m|\ge d$ for all $j\in A_+$ and $m\in M$. 
\end{assumption}

Intuitively, $E_{j,m}$ refers to users in $A$ whose orders of arrival rates are not far behind the arrival rate of $j$. Assumption \ref{ass: Arrivalcondition} says that, for each opted in user, there are enough other opted in users whose tastes are different from hers but have similar (or faster) arrival rate orders. It generalizes the user arrival model of \cite{hao2020adaptive} discussed in Section \ref{sec:haopapinireview}, i.e., i.i.d. random sampling according to a fixed distribution over $A$, where arrival rate orders are the same for all users in $A$: 

\begin{lemma}[Exponential inter-arrival times (equivalent to i.i.d. sampled arrivals)]\label{lemma:exponentialarrivals}
Suppose that each user $i \in A$ repeatedly arrives independently with $i.i.d.$ exponentially distributed inter-arrival times with parameter $\lambda_i$. Then $E_{j,m} = A$, and Assumption \ref{ass: Arrivalcondition} becomes $|A_+ \cap A_m|\ge d$ for $m\in M$, a reminiscent of condition  $|A_m|\ge d$ for $m\in M$ in Section \ref{sec:haopapinireview}. 
\end{lemma}

Lemma \ref{lemma:subgaussianarrivals} shows that Assumption \ref{ass: Arrivalcondition} also holds also for Sub-Gaussian arrivals (proof in Appendix \ref{proofs}). 

\begin{lemma}[Subgaussian inter-arrival times]\label{lemma:subgaussianarrivals}
Suppose that each user $i \in A$ repeatedly arrives independently with $i.i.d.$ 1-sub-Gaussian inter-arrival times with mean $\theta_i$. Then $E_{j,m} = A$, and Assumption \ref{ass: Arrivalcondition} becomes $|A_+ \cap A_m|\ge d$ for $m\in M$.
\end{lemma}

\subsection{Operationalizable condition for bounded regret}

Note that the condition described in Assumption \ref{ass: Arrivalcondition} of Section \ref{sec:arrivalModel} ($|E_{j,m}\cap A_+ \cap A_m|\ge d$) serves as a counterpart to the condition $|A_m|\ge d$ for $m\in M$ in Section \ref{sec:haopapinireview}; if we further assume that all user arrival rates are of the same order (which results in $E_{j,m} = A$ (e.g., Lemma \ref{lemma:exponentialarrivals} and \ref{lemma:subgaussianarrivals})), that condition becomes $|A_+ \cap A_m|\ge d$.
However, this condition cannot be verified, as $A_m$ is unknown: if $A_m$ were known, exploration would be unnecessary.

Theorem \ref{armrequirement} provides a path to an operationalizable condition: Assumption \ref{ass: Arrivalcondition} is highly likely to be satisfied if we are given a sufficiently large number of users in $A_+$ compared to the number of arms. Its proof is deferred to Appendix \ref{proofs}.  

\begin{theorem}\label{armrequirement}
Suppose that the optimal arms associated with users $\{m^*_j: j \in A\}$ are independently and uniformly distributed over $A$ and user arrival rates are of the same order.  
If $|A_+| \ge |M|d+\max \{|M|d , 4\left(|M|\ln |M|+ |M|\ln\frac{1}{\epsilon} +d\right) \}$ holds, then $
P(\{|A_m\cap A_+|\ge d\; \forall m\in M \}) \geq 1-\epsilon$.
\end{theorem}

\section{CFUCB algorithm}\label{sec:CFUCB}

We now introduce our main algorithm we call the Counterfactual-UCB (CFUCB) recommendation algorithm.
In a typical UCB-based algorithm (e.g., \cite{auer2002using} for the multi-armed bandit problem), 
each user forms a confidence interval solely based on her own experience, which one may call the \textit{self-experience based confidence interval}. For all users who opted in, the recommender not only knows the user's self-experience based confidence interval, but it can also construct a confidence interval based solely on other users' experiences. We call this the \textit{counterfactual confidence interval}.

\noindent \textit{\textbf{Self-experience based Confidence interval.}} \space \space Denote by
 $\overline{Y}_{j,m}(t)=\frac{\sum_{k=1}^{N_{j,m}(t) } Y_{j,m}(k)}{N_{j,m}(t)}$
the empirical mean reward of user $j$ on arm $m$,
and define the width 
${w}^{se}_{j,m}(t):=\sqrt{\frac{\ln N_j(t)}{ N_{j,m}(t)}}$. 
Defining $\overline{Y}_{j,m}(t)+{w}^{se}_{j,m}(t)$ as $ucb^{se}_{j,m}(t)$ and $\overline{Y}_{j,m}(t)-{w}^{se}0_{j,m}(t)$ as $lcb^{se}_{j,m}(t)$, the \textit{self-experienced confidence interval} is $CI^{se}_{j,m}(t):=[lcb^{se}_{j,m}(t), ucb^{se}_{j,m}(t)]$. 

\medskip
\noindent\textit{\textbf{Counterfactual Confidence interval.}}  \space Define ${A_m}(d, t):= \{j \in A_+:|\{i \in A_+: N_{i,m}(t)>N_{j,m}(t)\}|<d\}$.  This set includes the top $d$ users in $A_+$ for arm $m$ with all ties at the bottom being included. Since Assumption \ref{ass: Arrivalcondition} implies that $|A_+|\ge d+1$, $|{A_m}(d+1, t)|\ge d+1$. Let $E_{j,m}(t)$ be any arbitrarily chosen $d$-size subset 
 of ${A_m}(d+1, t)\setminus j$. From SCO (Assumption \ref{assumption: SCO}), we are given $\{a_{ji}\}_{i\in E_{j,m}}$ such that $\mu_{j,m} = \sum_{i\in E_{j,m}(t)}a_{ji} \mu_{i,m}$. 
Define $ \widehat{Y}_{j,m}(t):=\sum_{i\in E_{j,m}(t)} a_{ji} \overline{Y}_{i,m}(t)$ and call it the \textit{counterfactual mean reward of user $j$ for arm $m$}. The width  $w^{cf}_{j,m}(t)$ of the corresponding
counterfactual confidence interval is chosen as $w^{cf}_{j,m}(t):=\sqrt{\frac{2\ln d + 4\ln N_{j}(t)}{N_{j,m}^{\min}(t)/c_{m,t}^2}}$, where $c_{m,t} := \sum_{i\in E_{j,m}(t)} |a_{ji}|$, and $N_{j,m}^{\min}(t) := \min_{i\in E_{j,m}(t)} N_{i,m}(t)$. 
The \textit{counterfactual confidence interval} is defined as $CI^{cf}_{j,m}(t):=[lcb_{j,m}^{cf}(t), ucb^{cf}_{j,m}(t)]$, where
$ucb^{cf}_{j,m}(t) := \widehat{Y}_{j,m}(t)+w^{cf}_{j,m}(t)$ and $lcb^{cf}_{j,m}(t) := \widehat{Y}_{j,m}(t)-w^{cf}_{j,m}(t)$. 

For future references, let's restate the upper confidence bounds we defined above:
\begin{eqnarray}\label{ucbeq}
    ucb^{se}_{j,m}(t) := \overline{Y}_{j,m}(t) +w^{se}_{j,m}(t),\; ucb^{cf}_{j,m}(t):=\widehat{Y}_{j,m}(t) +w^{cf}_{j,m}(t).    
\end{eqnarray}

\noindent\textit{\textbf{The Counterfactual UCB (CFUCB) algorithm.}} Let $s_k$ be the time of the $k$th arrival from $A$, $a_k$ be the user that arrives at $s_k$, $m_k$ be the arm pulled at $s_k$, and $r_k$ be the corresponding reward. Note that $m_k$ and $r_k$ are known to the recommender if and only if user $a_k$ has opted in, i.e., $a_k\in A_+$.

\normalem
\begin{algorithm}
\label{cooperativeAlgo}
\DontPrintSemicolon
\SetKwInOut {Input}{input} \SetKwInOut{ Output} {output}
\caption{CFUCB Algorithm\label{IR}}
\For{$k=1, 2, \ldots$}{
Observe $s_k$ and $a_k$\;
    \If{$a_k\in A_+$ (i.e., $a_k$ is a user who opted in)}{
        \For{$m=1, 2, \ldots, |M|$}{
            Compute $ucb^{se}_{(a_k, m)}(s_k)$ (Self-experienced upper confidence bound) according to Eq  (\ref{ucbeq})\;
            Compute $ucb^{cf}_{(a_k, m)}(s_k)$ (counterfactual upper confidence bound) according to Eq (\ref{ucbeq})\;
            $\widetilde{ucb}_{(a_k, m)}(s_k) = \min( ucb^{se}_{(a_k, m)}(s_k), ucb^{cf}_{(a_k, m)}(s_k))$\; 
        }
        Set $m_k = \arg\min_{m\in M} \{\widetilde{ucb}_{(a_k,m)}(s_k)\}$\;
        Recommend user $a_k$ pull the arm $m_k$ and obtain $r_k$\;
        Store $Y_{(a_k, m_k)}(N_{a_k,m_k}(s_k))=r_k$ for later use in lines \textbf{4} and \textbf{5}
    }
}

\end{algorithm}
\ULforem

\section{Analysis of CFUCB}\label{sec:AnalysisCFUCB}

We first start by describing how the confidence intervals are chosen. Following the spirit of \cite{auer2002using}, they bound the violation probability by the inverse square of the total number of pulls at time $t$. The proofs are deferred to Section \ref{proofs}.

\begin{lemma}[\cite{auer2002using}]\label{leamma:origin_CI}

For $\epsilon \ge  \sqrt{\frac{4\ln N_j(t)}{ N_{j,m}(t)}}   $,  $P(|\overline{Y}_{j,m}(t) -\mu_{j,m}|>\epsilon) \leq N_j(t)^{-2}$. 

\end{lemma}

\begin{lemma}\label{lemma:CF_CI} Let $c_{m,t} := \sum_{i\in E_{j,m}(t)} |a_i^{(j)}|$. Then, for $\epsilon \ge \sqrt{\frac{2\ln d + 4\ln N_{j}(t)}{N_{j,m}^{\min}(t)/c_{m,t}^2}} $, $P(|\widehat{Y}_{j,m}(t) -\mu_{j,m}|>\epsilon) \leq N_j(t)^{-2}$. 
\end{lemma}

Lemmas \ref{Pre_WrongArmCond} and \ref{lemma:WrongArmCond} are the key results in this paper, in that they provide intuition of why bounded regret is achieved.

\begin{lemma}\label{Pre_WrongArmCond}
 If $CI^{se}_{i,n}(t)$ and ${CI}^{cf}_{i,n}(t)$ both include the true mean $\mu_{i,n}$ for all $i\in A$ and $n\in M$, user $j$ pulls arm $m$ only if $\min\left( 2\sqrt{\frac{4\ln N_j(t)}{ N_{j,m}(t)}} ,2 \sqrt{\frac{2\ln d + 4\ln N_{j}(t)}{N_{j,m}^{\min}(t)/c_{m,t}^2}} \right) \ge \Delta_{j,m}$, i.e.,  $N_{j,m}(t) \le  \frac{16\ln N_j(t)}{{\Delta_{j,m}}^2}$ and $N^{min}_{j,m}(t) \le  \frac{8 c^2_{m,t}  (\ln d + 2\ln N_{j}(t))}{{\Delta_{j,m}}^2}$.
\end{lemma}

\begin{lemma}\label{lemma:WrongArmCond} 
  If $CI^{se}_{i,n}(t)$ and ${CI}^{cf}_{i,n}(t)$ both include the true mean $\mu_{i,n}$ for all $i\in A$ and $n\in M$, then
  a user $j$ who arrives at time $t$ pulls a non-optimal arm $m$, i.e., one
with $\Delta_{j,m}>0$, only if

\begin{equation}\label{eqofWrongArmCond}
 \min_{E_{j,m}\cap A_+ \cap A_m}\{ N_{i}(t)- (\sum_{n\neq m}\frac{16  }{{\Delta_{i,n}}^2}  )\ln N_i(t) \} \le   \frac{8 c^2_{m,t}  (\ln d + 2\ln N_{j}(t))}{{\Delta_{j,m}}^2}.  
\end{equation}
\end{lemma}

See Section \ref{proofs} for their proofs. Lemma \ref{lemma:WrongArmCond} says that the inequality (\ref{eqofWrongArmCond}) is a necessary condition for a user $j$ to pull a non-optimal arm $m$.
Since LHS of growing like $N_i(t)$ increases far faster than the RHS of (\ref{eqofWrongArmCond}) growing only like $\ln N_j(t)$, the inequality will soon cease to hold for all non-optimal arms unless there exists a user $i$ in $E_{j,m}\cap A_+ \cap A_m$ with $N_i(t)$ that increases far slower than $N_j(t)$. Since there is no such $i\in E_{j,m}\cap A_+ \cap A_m$ (Assumption \ref{ass: Arrivalcondition}), only the optimal arm will be pulled afterwards except when the true mean $\mu_{i,n}$ is not included $CI^{se}_{i,n}(t)\cap {CI}^{cf}_{i,n}(t)$. A concentration inequality however assures the inclusion of the means in the confidence intervals with high probability. 

The following Proposition \ref{prop:BoundedRegret} formalizes the intuition provided above, showing that Assumptions \ref{assumption: SCO} and \ref{ass: Arrivalcondition} are indeed enough to achieve bounded expected regret. For the proof, refer to the Appendix \ref{proofs}.

\begin{proposition}\label{prop:BoundedRegret} Under Assumptions \ref{assumption: SCO} and \ref{ass: Arrivalcondition}, $E[Regret_{j}(T)]\le \sum_{m\in M} \Delta_{j,m} \left(\frac{\pi^2|A||M|}{6}+ \frac{k\pi^2|A|}{6}\right)$ for $j\in A_+$. Since this bound does not increase with $T$, the system expected regret $\sum_{j\in A}E[Regret_{j}(T)]$ is bounded. 
\end{proposition}

\section{Incentive to opt in and comply} Suppose that a user $i\in A$ only cares about the asymptotic order of the regret, and not its precise value. That is, user $i$ is indifferent between an $f(T)$ regret and an $g(T)$ regret if and only if $f(T) = \Theta (g(T))$. We say that such a user $i$ has \textit{asymptotic preference} (defined formally in Appendix \ref{sec:Detailincentive}). Would there be any incentive for the user $i$ to opt out or not follow the recommendation at any time? This is a dynamic game, and the question 
relates to whether opting in and following the recommendation constitutes a Subgame Perfect Nash Equilibrium (SPNE) \cite{fudenberg1991game}. If all the users of $A$ have asymptotically indifferent preferences, it is trivial that no user can strictly improve herself by opting out or not complying to recommendation since she already has $O(1)$ regret and there is no smaller order of regret that can be contemplated. Hence we have the following result:
\begin{theorem}\label{thm:SPNE}

Under Assumption \ref{ass: Arrivalcondition}, for users with asymptotically indifferent preferences, the strategy where every user opts in and complies is a Subgame Perfect Nash Equilibrium (SPNE).

\end{theorem}
The formal formulation of this game and result are provided in the Appendix \ref{sec:Detailincentive}.
\begin{remark}
    While Theorem \ref{thm:SPNE} posits our result as SPNE, our result is actually much more robust than SPNE. This is because we allow coalitional deviation as long as Assumption \ref{ass: Arrivalcondition} holds, i.e., $|E_{j,m}\cap A_+ \cap A_m|\ge d$ for all $j\in A_+$ and $m\in M$.
\end{remark}

\section{Simple simulation analysis}\label{sec:simulation}

\begin{figure}[ht]
\includegraphics[width=7cm]{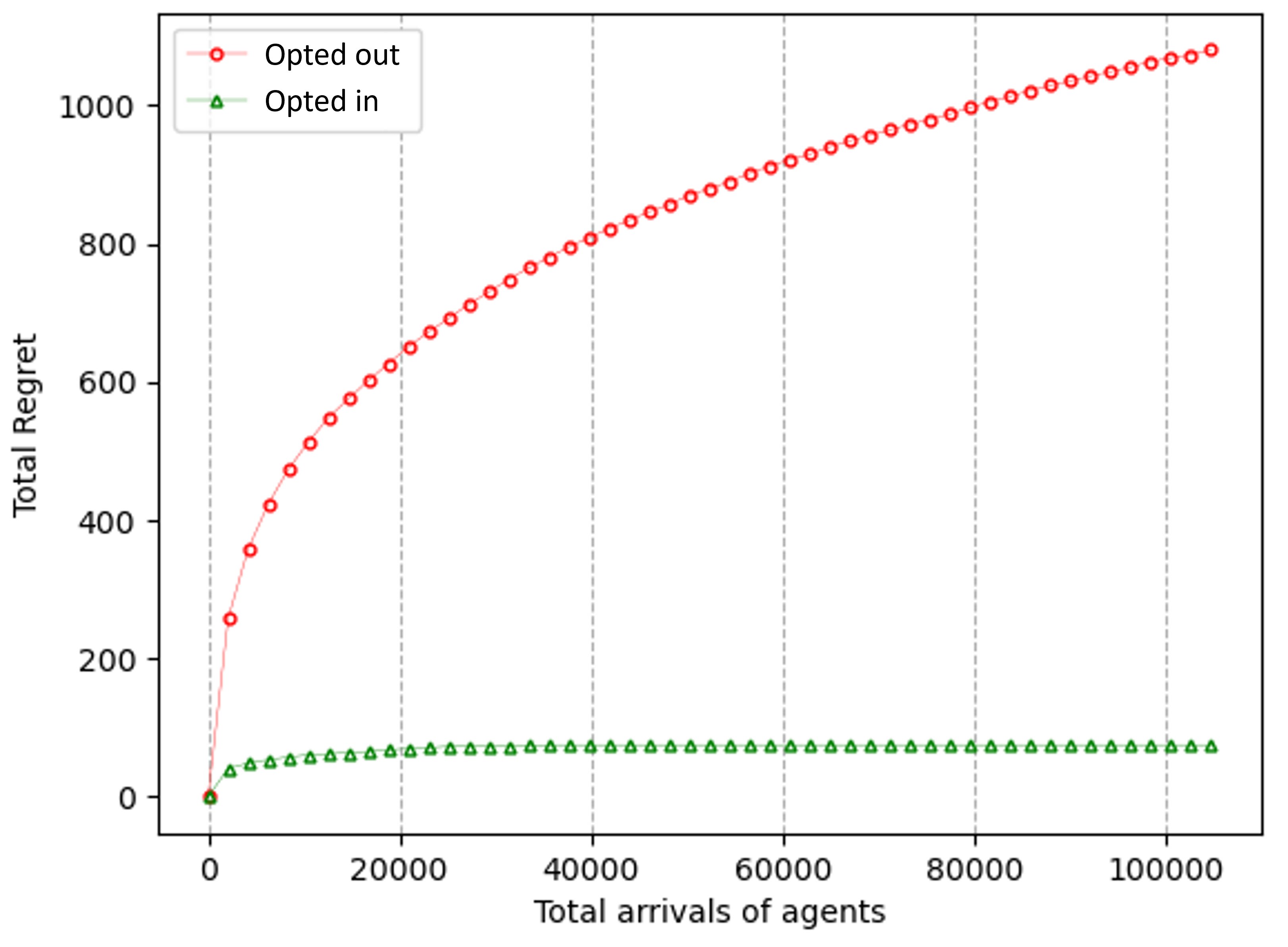}
\centering
\caption{The regret of the opted in users compared to that opted out users (200 users, 20 arms, feature vector dimension 5)} 
\label{CFUCB_UCB_Comparison}
\end{figure}

As we are suggesting a new practical setting that relaxes the 'knowledge of $\phi$' assumption (Assumption \ref{assumption: SCO}), our empirical simulation analysis can simply be devoted to verifying the theoretical results thus far. Specifically, we aim for the empirical demonstration of the opted-in users' $O(1)$ expected regret and the opted-out users' $O(\ln T)$ regret. Our SCM oracle  computes the coefficients using the user feature vectors. Our algorithm only knows about the coefficients.

In this experiment, we have 200 users repeatedly arriving to explore 20 arms. Each user independently arrives according to its own renewal process with positively truncated $i.i.d.$ Gaussian inter-arrival times. Both user and arm feature vectors (unknown) are randomly and uniformly generated as vectors on the surface of the $0$-centered unit sphere in $\mathbb{R}^5$. Rewards are generated according to the simplified disjoint model (Section \ref{sec:ourmodel}), i.e., the reward resulting from an arm pull is the inner product of the user's and the pulled arm's feature vectors plus i.i.d. $N(0,0.1)$ noise.

Figure \ref{CFUCB_UCB_Comparison} averages the results of ten experiments, with arrivals and feature vectors newly generated for each experiment. As can be seen, the regret graph for the opted-in users almost levels off by the time each user pulls each arm five times on average. In contrast, the average regret of the opted-out users grows logarithmically with $T$.

\section{Conclusions}\label{sec:Conclusion}

In this paper, we present a theoretical study addressing the challenges in applying recent bounded regret results \cite{hao2020adaptive, wu2020stochastic, papini2021leveraging} to practical recommender systems. These challenges encompass unobservable covariates, unknown linear representation functions, user arrival rates, and incentives to opt in. We present an algorithm that relies on a more practical assumption than the knowledge of linear representation functions, while still enabling bounded regret. This algorithm also allows other practical relaxations, including allowing very different orders of arrival rates among users. 

\section{Related works}\label{sec:related}
The issue of not knowing $\phi$ in linear contextual bandits has been studied in the representation learning literature. Recent studies \cite{tirinzoni2022scalable, tirinzoni2023complexity} examined the linear contextual bandit representation selection problem, i.e., learning to choose a good representation $\phi$ from a finite set of known representations $\Phi$. However, it remains challenging for practical recommender systems applications to assume the knowledge of $\Phi$. In \cite{du2020few}, they study $\phi$ learning problem beyond representation selection; however, their setting and results are not directly related to online learning settings. 

At the intersection of Synthetic Control Methods (SCM) and bandit methods, recent studies \cite{farias2022synthetically, agarwal2020synthetic, chen2023synthetic} have attempted to develop online learning versions of SCM. Compared to these works, our focus is not on developing a good SCM method itself, but on assuming existence of a good SCM method.  To the best of our knowledge, this work is the first to observe that what is achieved by SCM is a relaxation of what is assumed in the linear contextual bandit models. 

On the subject of incentive issues, there are many works on incentive constraints in coordinating exploration. \cite{chen2018incentivizing} studies Bayesian perspectives of incentivizing myopic users with a private context to explore, with the goal of achieving $O\left(\log ^{3}(T)\right)$ regret. 
\cite{immorlica2019bayesian} considers incentive-compatible exploration coordination in a setting opposite to ours: the context is private, but the mean reward associated with each arm is known. In this work, we illustrate that opting in (revelation of private information) is Subgame Perfect Nash Equilibrium (SPNE) and achieve $O(1)$ regret. 

\printbibliography

\newpage
\onecolumn

\section{Proofs}\label{proofs}

\begin{proof}[Proof of Lemma \ref{lemma:subgaussianarrivals}]
Let $P(N_i(t)< q_{ij,m}(N_j(t))) = h_{j,m}(t)$. Then 
\begin{eqnarray}
\hspace*{-0.5cm}\int h_{j,m}F_{S_j(n) } \notag
 \\
 &&\hspace*{-2.0cm}= \int_{[0, n(\theta^j-\epsilon))} h_{j,m}F_{S_j(n) }+\int_{[n(\theta^j-\epsilon), \infty)}  h_{j,m}F_{S_j(n) } \notag
\\&&\hspace*{-2.0cm}\le h_{j,m}(0^+)\times e^{-2n\epsilon^2}  + h_{j,m}(n(\theta^j-\epsilon))\times 1 \label{eq:14}
\\&&\hspace*{-2.0cm}=2|A|e ^{-2n\epsilon^2}  + h_m^{(j)}(  n(\theta^j-\epsilon))  \notag
\\&&\hspace*{-2.0cm}=2|A|(\exp ({-2n\epsilon^2} ) + \exp(-2\frac{(n(\theta^j-\epsilon)  -\lceil q_{ij,m}( \frac{n(\theta^j-\epsilon)  }{\theta^{j}-\epsilon^j}    )\rceil\theta_{\max})^2}{\lceil q_{ij,m}( \frac{n(\theta^j-\epsilon)  }{\theta^{j}-\epsilon^j}    )\rceil}) \notag
\\&&\hspace*{-1.0cm} + \exp(-2\frac{{\epsilon^j}^2}{\theta^j-\epsilon^j}n(\theta^j-\epsilon))) \label{eq:99}
\\&&\hspace*{-2.0cm}= 2|A|\left(2\exp ({-2n\epsilon^2} ) + \exp(-2\frac{(n(\theta^j-\epsilon)  -\lceil q_{ij,m}( n  )\rceil\theta_{\max})^2}{\lceil q_{ij,m}(n   )\rceil})\right) \notag
\\&&\hspace*{-1.0cm}(\mbox{for simplicity, we fix }\epsilon^j=\epsilon)\notag
\\&&\hspace*{-2.0cm}=O (\frac{1}{n^2}). \label{eq:O}
\end{eqnarray}
Above, \\
(\ref{eq:14}) holds because 
$P(S_n^{(j)}\le n(\theta^j-\epsilon))\le e^{-2n\epsilon^2} \mbox{ and since }g_m^{(j)} \mbox{ is a decreasing function}$,
\\
(\ref{eq:99}) holds because of Lemma \ref{lemma:middle1} below,
\\
(\ref{eq:O}) holds because
\begin{eqnarray}
&&\hspace*{-2.5cm} \left(n(\theta^j-\epsilon))  -\lceil q_{ij,m}(n )\rceil \theta_{\max}\right)^2 \ge \lceil q_{ij,m}(n)\rceil ^2 \mbox{ for all }n\ge N \mbox{ for some }N \label{eq:43}\\
&&\hspace*{-1.5cm}(\Rightarrow) (n(\theta^j-\epsilon))  -\lceil q_{ij,m}(n )\rceil \theta_{\max})^2 \ge \ln (n) \lceil q_{ij,m}(n )\rceil \mbox{ for all  }n\ge N
\label{eq:44}
\\
&&\hspace*{-1.5cm}(\Rightarrow) \exp(-2\frac{(n(\theta^j-\epsilon)  -\lceil q_{ij,m}(n  )\rceil\theta_{\max})^2}{\lceil q_{ij,m}(n  )\rceil}) = O(\frac{1}{n^2}), \notag    
\end{eqnarray}

\end{proof}

\begin{lemma}\label{lemma:middle1}
   $h_{j,m}(t):=P(N^{(i)}(t)< q_{ij,m}(N^{(j)}(t)))\le |A| \left( \exp(-2\frac{(t-\lceil q_{ij,m}( \frac{t}{\theta^{j}-\epsilon^j}   )\rceil \theta_{\max})^2}{\lceil q_{ij,m}( \frac{t}{\theta^{j}-\epsilon^j}  )\rceil }) +  \exp({-2\frac{{\epsilon^j}^2}{\theta^j-\epsilon^j}t   })\right)$.
\end{lemma}

\begin{proof}[Proof of Lemma \ref{lemma:middle1}]
     \begin{eqnarray}
&&  h_{j,m}(t):=P(N^{(i)}(t)< q_{ij,m}(N^{(j)}(t))) \notag
\\
&&\ \ \ \ \ \ =   \int P( N^{(i)}(t) < q_{ij,m}(n)) \; dF_{N^{(j)}(t)}(n)\notag
\\
&&\ \ \ \ \ \ \le   \int P( N^{(i)}(t) < q_{ij,m}(n)) \; dF_{N^{(j)}(t)}(n)\notag
\\
&&\ \ \ \ \ \ =   \int P(  S^{(i)}_{\lceil q_{ij,m}(n)\rceil)   } > t)   \; dF_{N^{(j)}(t)}(n)\notag
\\
&&\ \ \ \ \ \ = (\int_{[0, \frac{t}{\theta^j-\epsilon^j}]} P(  S^{(i)}_{\lceil q_{ij,m}(n)\rceil)   } > t)   \; dF_{N^{(j)}(t)}(n)\notag
\\
&&\ \ \ \ \ \ \ \ \ \ \ \ \ \ \ \ \ \ \ \ \;+\int_{(\frac{t}{\theta^j-\epsilon^j}, \infty)} P(  S^{(i)}_{\lceil q_{ij,m}(n)\rceil)   } > t)   \; dF_{N^{(j)}(t)}(n))\notag
\\
&&\ \ \ \ \ \  \stackrel{(c)}{\le}   \left( P(  S^{(i)}_{\lceil q_{ij,m}( \frac{t}{\theta^{j}-\epsilon^j}    )\rceil)   } > t)\times 1 +1 \times \exp({-2\frac{{\epsilon^j}^2}{\theta^j-\epsilon^j}t   })\right) \label{31}
\\
&&\ \ \ \ \ \ =  \left( \exp(-2\frac{(t-\lceil q_{ij,m}( \frac{t}{\theta^{j}-\epsilon^j}   )\rceil \theta_i)^2}{\lceil q_{ij,m}( \frac{t}{\theta^{j}-\epsilon^j}  )\rceil }) +  \exp({-2\frac{{\epsilon^j}^2}{\theta^j-\epsilon^j}t   })\right)\label{eq:34} \notag
\\
&&\ \ \ \ \ \ \le |A| \left( \exp(-2\frac{(t-\lceil q_{ij,m}( \frac{t}{\theta^{j}-\epsilon^j}   )\rceil \theta_{\max})^2}{\lceil q_{ij,m}( \frac{t}{\theta^{j}-\epsilon^j}  )\rceil }) +  \exp({-2\frac{{\epsilon^j}^2}{\theta^j-\epsilon^j}t   })\right).
\end{eqnarray}
Above, 
\begin{itemize}

    \item  $\theta_{\max} := \max_{i\in A} \theta_i.$
    \item The inequality (c) of (\ref{31}) holds because we apply left tail Hoeffding inequality, i.e.,  
\begin{eqnarray*}
P\left(S_{n}^{(j)}\leq n (\theta^j-\epsilon^j)\right)=P\left(N^{(j)}( n(\theta^j-\epsilon^j))\geq n\right)\leq e^{-2 n\epsilon^{j2} } \\
\Leftrightarrow P\left(N^{(j)}(t)\geq \frac{t}{\theta^j-\epsilon^j}\right)\leq e^{-2 \frac{\epsilon^{j2}   }{\theta^j-\epsilon^j}  t  }
\end{eqnarray*}
and 
$\lceil q_{ij,m}(n)\rceil$ is an increasing function of $n$.
\end{itemize}
\end{proof}

\begin{proof}[Proof of Theorem \ref{armrequirement}]
For simplicity, we denote $|A_+|=a$ and $|M|=b$. 
Let $I_m$ be the indicator random variable for the event $\{|A_m|<d+1\}$, and $I :=\sum_{m\in M} I_m$. What we want is to upper bound $P(I>0)$ by $\epsilon$.
Note that 
\begin{eqnarray}
&&\hspace*{-0.7cm} P(I>0) = P(I\ge 1) \notag
\\
&& \le  E[I] \ \ \ \mbox{ (because of Markov's inequality) }\notag
\\
&&  = bE[I_1] \notag\\\notag
&&  = bP(I_1=1) \\\notag
&&  = b \sum_{k=0}^{d} \binom{a}{k}\left(1-\frac{1}{b}\right)^{a-k}\left(\frac{1}{b}\right)^k
\\\notag
&& \le  b\sum_{k=0}^{d} \binom{a}{k}\left(1-\frac{1}{b}\right)^{a-d}\left(\frac{1}{b}\right)^k
\\\notag
&& \le  b \sum_{k=0}^{d} \frac{a^k}{k!}{\exp}(-\frac{a-d}{b})\left(\frac{1}{b}\right)^k \mbox{ (because } \binom{a}{k} \leq \frac{a^{k}}{k !} \mbox{, and } 1+x\le e^x)
\\
&& =  {\exp}\left(\frac{d}{b}\right)  \sum_{k=0}^{d} \frac{1}{k!} \left(\frac{a}{b}\right)^k{\exp}\left(-\frac{a}{b}\right) \notag
\\
&& =  b {\exp}\left(\frac{d}{b}\right)P\left(Z\le d\right) \mbox{, where }Z\sim Poi(\frac{a}{b}) \notag
\\
&& \stackrel{(a)}{\le}  b {\exp}\left(\frac{d}{b}\right)\exp\left(-\frac{1}{2}\frac{b}{a}\frac{\left(a-bd\right)^2}{b^2}\right)\label{eq:5}
\\
&&=b {\exp}\left(\frac{1}{b}\left(d-\frac{\left(a-bd\right)^2}{2a} \right)\right)\notag
\\
&&= {\exp}\left(\ln b-\frac{1}{b}\left(\frac{\left(a-bd\right)^2}{2a}-d \right)\right).\label{eq:6}
\end{eqnarray}
Above, the inequality $(a)$ of (\ref{eq:5}) holds because $Z\sim Poisson(\lambda)$,  $\operatorname{Pr}[Z \leq \lambda-x] \leq e^{-\frac{x^{2}}{2 \lambda} }$ for $0\leq x\leq\lambda$, where in our case $\frac{a}{b}\ge d$ as assumed, $\lambda = \frac{a}{b}$, $\lambda-x  = d$ and $x=\frac{a}{b}-d=\frac{a-bd}{b}$).

Let us further assume that $a \ge (1+\eta) bd$. Now
\begin{eqnarray}
    &&\hspace*{-0.7cm} a \ge (1+\eta) bd \notag
    \\
    &&(\Leftrightarrow) \ \ \ (1+\eta)(a-bd) \ge (1+\eta)a-a = \eta a \notag
    \\
    &&(\Leftrightarrow) \ \ \ a \le (a-bd) \frac{(1+\eta)}{\eta}.\label{eq:7}
\end{eqnarray}
Then,
\begin{eqnarray}
    &&\hspace*{-0.7cm} P(I>0)\le\epsilon \notag
    \\
    &&(\Leftarrow) \ \ \ {\exp}\left(\ln b-\frac{1}{b}\left(\frac{\left(a-bd\right)^2}{2a}-d \right)\right) \le \epsilon \ \ \mbox{  (because of (\ref{eq:6}))}\notag
    \\
    &&(\Leftrightarrow) \ \ \  \exp\left(-\frac{(\frac{\left(a-bd\right)^2}{2a}-d )-b\ln b}{b}\right) \le \epsilon \notag
    \\
    &&(\Leftrightarrow) \ \ \  \frac{\left(a-bd\right)^2}{2a}\ge b\ln b+ b\ln\frac{1}{\epsilon}+d \notag
    \\
    && (\Leftarrow) \ \ \  a-bd \ge \frac{2(1+\eta)}{\eta}\left(b\ln b+ b\ln\frac{1}{\epsilon} +d\right) \ \ \mbox{  (because of (\ref{eq:7}))}
\end{eqnarray}
\end{proof}

\begin{proof}[Proof of Lemma \ref{leamma:origin_CI}.]
This follows from the 1-sub-Gaussian tail bound $P(|\overline{Y}_{j,m}(t) -\mu_{j,m}|>\epsilon)\le 2exp( {-N_{i,m}(t)\epsilon^2/2})$. 
Since we want to upper bound $P(|\overline{Y}_{j,m}(t) -\mu_{j,m}|>\epsilon)\le$ by $N_j(t)^{-2}$, the value of $\epsilon$ that renders 2$exp( {-N_{i,m}(t)\epsilon^2}/2) \le N_j(t)^{-2}$ will suffice. This yields $\epsilon \ge  \sqrt{\frac{4\ln N_j(t)}{ N_{j,m}(t)}}$.
\end{proof}

\begin{proof}[Proof of Lemma \ref{lemma:CF_CI}.]
\begin{eqnarray}
&&P(|\widehat{Y}_{j, m}(t) -\mu_{j,m}|>\epsilon) \notag\\
&&\ \ \ \ \ \ =  1-P(|\widehat{Y}_{j,m}(t)-\mu_{j,m}|\le\epsilon)\notag
\\
&& \ \ \ \ \ \ \le 1- \Pi_{i\in{E_{j,m}(t)}}P( |a_{ji}| |  {Y}_{i,m}(t) -\mu_{i,m}| \le  |a_i^{(j)}| \frac{  \epsilon}{c_{m,t}})\notag
\\
&& \ \ \ \ \ \ = 1- \Pi_{i\in E_{j,m}(t)}(1-P(|  {Y}_{i,m}(t)  -\mu_{i,m}|> \frac{\epsilon}{c_{m,t}}))\notag
\\
&& \ \ \ \ \ \ \le 1-   \Pi_{i\in E_{j,m}(t)}( (1- \exp( \frac{-N_{i,m}(t)\epsilon^2}{2c_{m,t}^2}))) \notag
\\
&& \ \ \ \ \ \ \le 
1- \Pi_{i\in E_{j,m}(t)} (1- \exp( \frac{-N_{j,m}^{\min}(t) \epsilon^2}{2c_{m,t}^2}) \; \; (\because N_{j,m}^{\min}(t) := \min_{i\in E_{j,m}(t)} N_{i,m}(t))\notag
\\
&& \ \ \ \ \ \ 
= 1-(1- \exp( \frac{- N_{j,m}^{\min}(t) \epsilon^2}{2c_{m,t}^2}) )^d \notag
\\
&& \ \ \ \ \ \ \le d \exp( \frac{-N_{j,m}^{\min}(t) \epsilon^2}{2c_{m,t}^2}). \notag
\end{eqnarray}

Therefore, $\epsilon\geq \sqrt{\frac{2\ln (d/\delta)}{N_{j,m}^{\min}(t)/c_{m,t}^2}}$ implies $P(|\widehat{Y}_{j,m}(t) -\mu_{j,m}|>\epsilon)\leq \delta$.

Since we want CI with $\delta = 1/N_j(t)^{-2}$ following the spirit of \cite{auer2002finite}, CI with width $\sqrt{\frac{2\ln d + 4\ln N_{j}(t)}{N_{j,m}^{\min}(t)/c_{m,t}^2}}$ works.

\end{proof}

\begin{proof}[Proof of Lemma \ref{Pre_WrongArmCond}]
Denote the optimal arm for user $j$ as arm $m^*_j$. 
According to Algorithm \ref{cooperativeAlgo}, $\{$user $j$ pulls arm $m\}\subseteq\{\widetilde{\text{ucb}}_{j,m}(t) \ge\widetilde{\text{ucb}}_{j,m^*_j}(t)\}$. Note that $\widetilde{\text{lcb}}_{j,m}(t)\le \mu_{j,m}\le \widetilde{\text{ucb}}_{j,m}(t)$ and $\widetilde{\text{lcb}}_{j,m^*_j}(t)\le \mu_{j,m^*_j}\le \widetilde{\text{ucb}}_{j,m^*_j}(t)$ holds according to the assumption that under the assumption that all true means are within CIs. Therefore, $\{$user $j$ pulls arm $m\}\subseteq\{\widetilde{\text{lcb}}_{j,m}(t)\le\mu_{j,m}, \mu_{j,m} \le \mu_{j,m^*}, \mu_{j,m^*}\le \widetilde{\text{ucb}}_{j,m^*_j}(t), \widetilde{\text{ucb}}_{j,m^*_j}(t)\le \widetilde{\text{ucb}}_{j,m}(t)\}$ $=$  $\{\widetilde{\text{lcb}}_{j,m}(t)\le \mu_{j,m}\le\mu_{j,m^*_j}\le \widetilde{\text{ucb}}_{j,m}(t)\} = \{\mu_{j,m}, \mu_{j,m^*_j}\in  CI^{se}_{j,m}(t)\cap {CI}^{cf}_{j,m}(t)\}$. Note that $\{\mu_{j,m}, \mu_{j,m^*_j}\in  CI^{se}_{j,m}(t)\cap {CI}^{cf}_{j,m}(t)\}\subseteq\{\min (2w^{se}_{j,m}(t),2 w^{cf}_{j,m}(t)) \ge  \Delta_{j,m}\}$. Therefore, under the assumption that all true means are within CIs, user $j$ pulls arm $m$ only if $\min (2w^{se}_{j,m}(t),2 w^{cf}_{j,m}(t)) \ge  \Delta_{j,m}$ holds. Combining this with Lemma \ref{leamma:origin_CI} and \ref{lemma:CF_CI} yields the result.
\end{proof}

\begin{proof}[Proof of Lemma \ref{lemma:WrongArmCond}.]
Fix user $j$ and arm $m$. Note that for any arm $i\in A$, $N_{i,m}(t) = N_i(t)-\sum_{n \in M\setminus m} N_{i,n}(t)$. Let $t^n$ be the last time prior to $t$ at which a non-optimal arm $n$ is played by user $i$. Then $N_{i,n}(t) = N_{i,n}(t^n) \le \frac{16\ln N_i(t^n)}{{\Delta_{i,n}}^2} \leq \frac{16\ln N_i(t)}{{\Delta_{i,n}}^2}$ holds 
by Lemma \ref{Pre_WrongArmCond}. Therefore, for user $i \in A_m$, for arm $m$, $N_{i,m}(t) \ge N_i(t)- (\sum_{n\neq m  }\frac{16}{{\Delta_{i,n}}^2})\ln N_i(t)$.  
By the Assumption (\ref{ass: Arrivalcondition}), $|E_{j,m}\cap A_+\cap A_m|\ge d$ holds, and therefore $N_{j,m}^{\min}(t) \ge N_{i,m} (t)$ holds for some $i \in A_{j,m}$. Therefore,
$N_{j,m}^{\min}(t) \ge N_{i,m} (t)\ge N_i(t)- (\sum_{n\neq m}\frac{16}{{\Delta_{i,n}}  ^2}  )\ln N_i(t)$ for some $i \in E_{j,m}\cap A_+\cap A_m$. That is,  $N_{j,m}^{\min}(t)\ge \min_{E_{j,m}\cap A_+\cap A_m}\{ N_i(t)- (\sum_{n\neq m}\frac{16      }{{\Delta_{i,n}}  ^2}  )\ln N_i(t) \}$. Substituting this into $N^{(\min)}_{j,m}(t) \le  \frac{8 c^2_{m,t} (\ln d +2 N_j(t))}{{\Delta_{j,m}}^2}$ from Lemma \ref{Pre_WrongArmCond}, it can be seen that arm $m$ is pulled by user $j$ only when $\min_{E_{j,m}\cap A_+\cap A_m}\{ N_i(t)- (\sum_{n\neq m}\frac{16   }{{\Delta_{i,n} }^2}  )\ln N_i(t) \} \le  \frac{8 c^2_{m,t} (\ln d +2 N_j(t))}{{\Delta_{j,m}}^2}$.
\end{proof}

\begin{proof}[\textbf{Proof of Proposition \ref{prop:BoundedRegret}}.] 
Let $G_{j,m}:=\{$User $j$ arrives at time $t$ and pulls a non-optimal arm $m\}$ and $V_t :=\{\mu_{i,n}\in  CI^{se}_{i,n}(t)\cap CI^{cf}_{i,n}(t) \ \forall i\in A_m, n\in M\}$ as $V(t)$. Let $P(G_{j,m}(t)|V(t)) = g_{j,m}(t)$. Then,

\begin{eqnarray}
&&E[Regret_j(T)]= \sum_{m\in M\setminus  m^*_j} \Delta_{j,m} E[\text{\# of user $j$'s non-optimal arm $m$ pulls before $T$}] \notag
\\
&&\ \ \ \ \ \ \ = \sum_{m\in M\setminus  m^*_j} \Delta_{j,m} \sum_{n=1}^{\infty} E[1_{G_{j,m}(S_{n,j})}1_{S_{n,j}\le T}]\notag
\\
&&\ \ \ \ \ \ \ = \sum_{m\in M\setminus  m^*_j} \Delta_{j,m} \sum_{n=1}^{\infty} E[E[1_{G_{j,m}(S_{n,j})}1_{S_{n,j}\le T}|S_{n,j}]]     \notag
\\
&&\ \ \ \ \ \ \ =  \sum_{m\in M\setminus  m^*_j} \Delta_{j,m} (\sum_{n=1}^{\infty} E[E[1_{G_{j,m}( S_{n,j}) }1_{S_{n,j}\le T}|V(S_{n,j})^c,S_{n,j}  ]P(V(S_{n,j})^c|S_{n,j})+    \notag
\\
&&\ \ \ \ \ \ \ \ \ \ \ \ E[1_{G_{j,m}(S_{n,j})}1_{S_{n,j}\le T}|V(S_{n,j}), S_{n,j}] P(V(S_{n,j})|S_{n,j})])    \notag
\\
&&\ \ \ \ \ \ \ \le \sum_{m\in M\setminus  m^*_j} \Delta_{j,m} \left(\sum_{n=1}^{\infty} E[P(V(S_{n,j})^c|S_{n,j})]+\sum_{n=1}^{\infty} E[E[1_{G_{j,m}(S_{n,j})}1_{S_{n,j}\le T}|V(S_{n,j}), S_{n,j}]]\right)     \notag
\\
&&\ \ \ \ \ \ \ \le \sum_{m\in M\setminus  m^*_j} \Delta_{j,m} \left(\frac{\pi^2|A||M|}{6}+ \sum_{n=1}^{\infty} E[E[1_{G_{j,m}(S_{n,j})}|V(S_{n,j}), S_{n,j}] \right)      \notag
\\
&&\ \ \ \ \ \ \ \sum_{m\in M\setminus  m^*_j} \Delta_{j,m} \left(\frac{\pi^2|A||M|}{6}+ \sum_{n=1}^{\infty} E[P({G_{j,m}(S_{n,j})} |V(S_{n,j}),S_{n,j}] \right)      \notag
\\
&&\ \ \ \ \ \ \ \ =\sum_{m\in M\setminus  m^*_j} \Delta_{j,m} \left(\frac{\pi^2|A||M|}{6}+ \sum_{n=1}^{\infty} \int_0^{+\infty} g_{j,m}(t) dF^{(n)}_j(t) \right)      \notag
\\
&&\ \ \ \ \ \ \ \ \stackrel{(a)}{\le} \sum_{m\in M\setminus  m^*_j} \Delta_{j,m} \left(\frac{\pi^2|A||M|}{6}+ \sum_{n=1}^{\infty} \int_0^{+\infty} \sum_{i\in E_{j,m}\cap A_+ \cap A_m } P(N_i(t)\le q_{ij,m}(N_i(t)))\;\; dF^{(n)}_j(t) \right) \label{eq:bddproof(a)}
\\
&&\ \ \ \ \ \ \ \ \ = \sum_{m\in M\setminus  m^*_j} \Delta_{j,m} \left(\frac{\pi^2|A||M|}{6}+ \sum_{i\in E_{j,m}\cap A_+ \cap A_m }\sum_{n=1}^{\infty}  \int_0^{+\infty} P(N_i(t)\le q_{ij,m}(N_i(t)))\;\; dF^{(n)}_j(t) \right)      \notag
\\
&&\ \ \ \ \ \ \ \ \ \le \sum_{m\in M\setminus  m^*_j} \Delta_{j,m} \left(\frac{\pi^2|A||M|}{6}+ \sum_{i\in E_{j,m}\cap A_+ \cap A_m }\sum_{n=1}^{\infty}  k \frac{1}{n^2}\right)  \;\; (\because \text{Assumption \ref{ass: Arrivalcondition}})    \notag
\\
&&\ \ \ \ \ \ \ \ \ \le \sum_{m\in M\setminus  m^*_j} \Delta_{j,m} \left(\frac{\pi^2|A||M|}{6}+ \frac{k\pi^2|A|}{6}\right)=\sum_{m\in M} \Delta_{j,m} \left(\frac{\pi^2|A||M|}{6}+ \frac{k\pi^2|A|}{6}\right).       \notag
\end{eqnarray}

Above, inequality (a) of equation (\ref{eq:bddproof(a)}) is from
\begin{eqnarray}
   &&g_{j,m}(t)\le P(\{\min_{E_{j,m}\cap A_+ \cap A_m}\{ N_{i}(t)- (\sum_{n\neq m}\frac{16  }{{\Delta_{i,n}}^2}  )\ln N_i(t) \} \le   \frac{8 c^2_{m,t}  (\ln d + 2\ln N_{j}(t))}{{\Delta_{j,m}}^2}\})\;\; (\because \text{Lemma \ref{lemma:WrongArmCond}})\notag
   \\
   &&\ \ \ \ \ \ \ \ \le \sum_{i\in E_{j,m}\cap A_+ \cap A_m } P(\{N_{i}(t)- (\sum_{n\neq m}\frac{16  }{{\Delta_{i,n}}^2}  )\ln N_i(t)  \le \frac{8 c^2_{m,t}  (\ln d + 2\ln N_{j}(t))}{{\Delta_{j,m}}^2}\})\notag
   \\ 
   &&\ \ \ \ \ \ \ \ \le \sum_{i\in E_{j,m}\cap A_+ \cap A_m } P(N_i(t)\le q_{ij,m}(N_i(t)))\;\; (\because \text{Lemma \ref{Lambert}})\notag
\end{eqnarray}

\end{proof}

\section{Function \text{$q_{ij,m}$} in Section \ref{sec:arrivalModel}: details}\label{sec:qij}

\begin{lemma} \label{Lambert}
For $A,B,C>0$, $Ay-B\ln y < C\ln(\frac{x}{d})$ is satisfied only if $y < -\frac{B}{A}  \mathcal{W}_{-1}\left( -\frac{A}{B} (\frac{x}{d})^{-\frac{C}{B}} \right)$, where $\mathcal{W}_{-1}$ denotes the lower branch of the Lambert $W$-function \cite{corless1996lambertw}.
\end{lemma}
\begin{proof}[Proof of Lemma \ref{Lambert}]
For $A,B,C>0$, $\frac{A}{C}y-\frac{B}{C}\ln y < \ln(\frac{x}{d}) \iff y^{-\frac{B}{C}} e^{\frac{A}{C}y} < (\frac{x}{d}) \iff ye^{-\frac{A}{B}y} > (\frac{x}{d})^{-\frac{C}{B}} \iff -\frac{A}{B}y e^{-\frac{A}{B}y}< -\frac{A}{B} (\frac{x}{d})^{-\frac{C}{B}} \iff -\frac{B}{A}  \mathcal{W}_{0}\left( -\frac{A}{B} (\frac{x}{d})^{-\frac{C}{B}} \right) < y<-\frac{B}{A}  \mathcal{W}_{-1}\left( -\frac{A}{B} (\frac{x}{d})^{-\frac{C}{B}} \right)$ where $\mathcal{W}_{0}$ denotes the principal branch of the Lambert $W$-function. Therefore, $Ay-B\ln y < C\ln(\frac{x}{d})$ holds only if $y < -\frac{B}{A}  \mathcal{W}_{-1}\left( -\frac{A}{B} (\frac{x}{d})^{-\frac{C}{B}} \right)$. 
\end{proof}
In the present case, $y=N_i(t)$, $x=N_j(t)$, $A=1$, $B=\sum_{n\neq m}\frac{16 }{{\Delta_{i,n}}  ^2}  $ and $C=\frac{16 c^2_{m,t}}{{\Delta_{j,m}}  ^2}$. Define $q_{ij,m}$ as $q_{ij,m}(x)=-\frac{B}{A}  \mathcal{W}_{-1}\left( -\frac{A}{B} (\frac{x}{d})^{-\frac{C}{B}} \right)$ where we use the above parameter values. One can easily check that $\frac{B}{A}  \mathcal{W}_{-1}\left( -\frac{A}{B} x^{-\frac{C}{B}}\right)$ is a function growing faster than $\ln x$ and slower than $x$.

\section{Incentive considerations}\label{sec:Detailincentive}

\subsection{Sequential game description}

The CFUCB Algorithm \ref{cooperativeAlgo} can be posited as a game $G = (A, M, \{\{S_i(n)\}_{n\in \mathbb{N}}\}_{i\in A}, \{\mathbf{x}_i\}_{i\in A}, \{\beta_m\}_{m\in M}, \Gamma)$. It  is defined as an $|A|$-player infinite horizon sequential game, where
\begin{itemize}[leftmargin=0.2in]
    \item [-] $A$ denotes the index set of users and $M$ denote the index set of arms.
    \item [-] $\{\{S_{i}(n)\}_{n\in \mathbb{N}}\}_{i\in A}$ denotes the arrival time processes of users in $A$
    \item [-] $\Gamma$ denotes the counterfactual UCB sharing mechanism (we describe below).
\end{itemize}
$G$ is a sequential game \cite{fudenberg1991game} where an arrival of any user in $A$ is one stage of the game. At the beginning of the game, which we call epoch $0$, each user $i$ is asked to report its feature vector $\mathbf{x}_i$. (Of course, it can refuse to report it by opting out at time $0$). At each epoch $k$, 

\begin{enumerate}[wide]
\item [1)] A user we denote by $a_k\in A$ arrives. The recommender observes $a_k$. 
\item [2)] If and only if $a_k \in A_+$, the recommender calculates the counterfactual UCBs $\{{ucb}^{cf}_{a_{k},m}(s_k)\}_{m\in M}$ according to Equation (\ref{ucbeq}) and lets user $a_k$ know the counterfactual UCBs. 

\item[\textbf{\textit{Remark.}}] \textit{$\Gamma$ of game $G$, the counterfactual-UCB sharing mechanism, is formally defined as a function that maps the previous history of reports the recommender has at $k$, $\{a_l, (\widehat{m}_l, \widehat{r}_l)\}_{l=1}^{k-1}$, into $\{{ucb}^{cf}_{a_{k}, m}(s_k)\}_{m\in M}$.}
\item [3)] After receiving $\{{ucb}^{cf}_{a_{k}, m}(s_k)\}_{m\in M}$ from the recommender, the user calculates $\widetilde{ucb}_{a_{k}, m}(s_k)$ for all $m\in M$ according to $\widetilde{ucb}_{a_{k}, m}(s_k) = \min( ucb^{se}_{a_{k}, m}(s_k), {ucb}^{cf}_{a_{k}, m}(s_k)$. (Note that user $a_{k}$ can calculate $\{ucb^{se}_{a_{k}, m}(s_k)\}_{m\in M}$ by only using it's own pulling history, which is private information.)
\item [4)] User $a_k$ then pulls arm $m_k := \arg\min_{m\in M} \{\widetilde{ucb}_{a_{k}, m}(s_k)\}$ and observes a reward that we denote by $r_k$. 
\item [5)] According to its reporting strategy, user generates its report $(\widehat{m}_k, \widehat{r}_k)$ from the truth $({m}_k, {r}_k)$ and sends it to the recommender. 
\item [6)] The recommender receives $(\widehat{m}_k, \widehat{r}_k)$ and stores it.

\end{enumerate}

\subsection{Incentive analysis}

We denote by $\sigma_i$ the strategy of user $i$ of never opting out and complying to recommendation at any of its arrivals. We 
define $\sigma=\times_{i\in A} \sigma_i$ as the \textit{strategy profile} corresponding to each user $i\in A$ following $\sigma_i$.

The 
strategy 
{profile} 
where every $i\in A$ chooses $\sigma_i$ is defined as $\sigma$. When no user ever violates the two assumptions, the outcome of $\sigma$ and Algorithm $\ref{cooperativeAlgo}$ are the same. Corollary \ref{Gameresult}, which is an immediate result of Proposition \ref{prop:BoundedRegret}, formally states this observation.

\begin{corollary}\label{Gameresult} 
Under Assumption \ref{assumption: SCO} and \ref{ass: Arrivalcondition}, under the 
strategy profile 
$\sigma$, every user's expected regret is bounded.
\end{corollary}
Now we formally define the notion of
``asymptotically indifferent users". Given the game $G$ and some 
strategy profile 
$\sigma$, after playing the game up to time $T$, denote the regret of user $i$ up to time $T$ by $Regret^{(i)}_{\sigma}(T)$. Suppose that for each $i\in A$, we are able to achieve $E[Regret^{(i)}_{\sigma}(T)] = O(f_\sigma^{(i)}(T)]$
for some function $f_{\sigma}^{(i)}$.
Denoting the set of all possible strategy profiles 
$\Sigma$, we say that an user $i$ has an \textit{asymptotically indifferent preference} if its preference can be described by a complete and transitive preference relation $\succeq_i$ on $\Sigma$ such that $\sigma \succeq_i {\sigma'}$ if and only if $f^{(i)}_{\sigma}(T) = O(f^{(i)}_{\sigma'} (T))$. We say that $\sigma$ is strictly preferred to $\sigma'$ by user $i$ if $\sigma \succeq_i {\sigma'}$ but not ${\sigma'} \succeq_i \sigma$.

\begin{corollary}\label{SPNE}
Suppose that all the users in $A$ have asymptotically indifferent preferences. Then $\sigma$ constitutes a Subgame Perfect Nash Equilibrium for the game $G = (A, M, \{\{S_i(n)\}_{n\in \mathbb{N}}\}_{i\in A}, \{\mathbf{x}_i\}_{i\in A}, \{\beta_m\}_{m\in M}, \Gamma)$.
\end{corollary}
\begin{proof}[Proof of Corollary \ref{SPNE}] This result is immediate from Corollary \ref{Gameresult}, in that (i) no other strategy profile can be strictly preferred to $\sigma$ by any user with aymptotically indifferent preference; (ii) $\sigma$ already achieves bounded regret, i.e., $O(1)$, for all the users, and (iii) thus cannot be improved in terms of asymptotically indifferent preference.
\end{proof}

\end{document}